\documentclass[11pt,oneside,reqno]{article}
\pdfoutput=1
\usepackage[top=1in, bottom=1in, left=1in, right=1in]{geometry}
\usepackage{amsmath,amsthm,amssymb}
\usepackage{bbm,tikz,tabularx}
\usepackage{mathrsfs,mathtools}
\usepackage{enumitem,pifont}
\usepackage{color}
\newcommand{\cmark}{\ding{51}}%
\newcommand{\xmark}{\ding{55}}%

\usepackage{wasysym,makecell}
\usepackage[authoryear,round]{natbib}
\newcommand{\stkout}[1]{\ifmmode\text{\sout{\ensuremath{#1}}}\else\sout{#1}\fi}
\newcommand{\ignore}[1]{}

\usepackage{url}

\usepackage{etoolbox} 

\makeatletter

\newif\if@noindent

\def\@aftertheorem{%
  \@noindenttrue
  \everypar{%
    \if@noindent
      \@noindentfalse\clubpenalty\@M\setbox\z@\lastbox
    \else
      \clubpenalty \@clubpenalty\everypar{}
    \fi}}

\let\aftertheorem\@aftertheorem

\makeatother

\theoremstyle{plain}
\newtheorem{theorem}{Theorem}[section]

\newtheorem{lemma}[theorem]{Lemma}
\newtheorem{corollary}[theorem]{Corollary}

\theoremstyle{definition}
\newtheorem{remark}[theorem]{Remark}

\AfterEndEnvironment{theorem}{\aftertheorem}
\AfterEndEnvironment{proposition}{\aftertheorem}
\AfterEndEnvironment{lemma}{\aftertheorem}
\AfterEndEnvironment{corollary}{\aftertheorem}
\AfterEndEnvironment{remark}{\aftertheorem}
\AfterEndEnvironment{definition}{\aftertheorem}
\AfterEndEnvironment{example}{\aftertheorem}

\usepackage[titles]{tocloft}
\usepackage{xcolor} 

\newcommand*{\colorboxed}{}
\def\colorboxed#1#{%
  \colorboxedAux{#1}%
}
\newcommand*{\colorboxedAux}[3]{%
  \begingroup
    \colorlet{cb@saved}{.}%
    \color#1{#2}%
    \boxed{%
      \color{cb@saved}%
      #3%
    }%
  \endgroup
}

\usepackage[linktocpage,colorlinks,linkcolor=purple,anchorcolor=blue,citecolor=purple,urlcolor=black,pagebackref]{hyperref}

\renewcommand*{\backref}[1]{\ifx#1\relax \else Page #1 \fi}
\renewcommand*{\backrefalt}[4]{%
  \ifcase #1 \footnotesize{(Not cited.)}%
  \or        \footnotesize{(Cited on page~#2.)}%
  \else      \footnotesize{(Cited on pages~#2.)}%
  \fi
}

\def\be#1{\begin{equation*}#1\end{equation*}}
\def\ben#1{\begin{equation}#1\end{equation}}

\def\besn#1{\begin{equation}\begin{split}#1\end{split}\end{equation}}

\def\ba#1{\begin{align*}#1\end{align*}}
\def\ban#1{\begin{align}#1\end{align}}

\def\given{\typeout{Command 'given' should only be used within bracket command}}
\newcounter{@bracketlevel}
\def\@bracketfactory#1#2#3#4#5#6{
\expandafter\def\csname#1\endcsname##1{%
\addtocounter{@bracketlevel}{1}%
\global\expandafter\let\csname @middummy\alph{@bracketlevel}\endcsname\given%
\global\def\given{\mskip#5\csname#4\endcsname\vert\mskip#6}\csname#4l\endcsname#2##1\csname#4r\endcsname#3%
\global\expandafter\let\expandafter\given\csname @middummy\alph{@bracketlevel}\endcsname
\addtocounter{@bracketlevel}{-1}}%
}
\def\bracketfactory#1#2#3{%
\@bracketfactory{#1}{#2}{#3}{relax}{1mu plus 0.25mu minus 0.25mu}{0.6mu plus 0.15mu minus 0.15mu}
\@bracketfactory{b#1}{#2}{#3}{big}{1mu plus 0.25mu minus 0.25mu}{0.6mu plus 0.15mu minus 0.15mu}
\@bracketfactory{bb#1}{#2}{#3}{Big}{2.4mu plus 0.8mu minus 0.8mu}{1.8mu plus 0.6mu minus 0.6mu}
\@bracketfactory{bbb#1}{#2}{#3}{bigg}{3.2mu plus 1mu minus 1mu}{2.4mu plus 0.75mu minus 0.75mu}
\@bracketfactory{bbbb#1}{#2}{#3}{Bigg}{4mu plus 1mu minus 1mu}{3mu plus 0.75mu minus 0.75mu}
}
\bracketfactory{clc}{\lbrace}{\rbrace}
\bracketfactory{clr}{(}{)}
\bracketfactory{cls}{[}{]}
\bracketfactory{abs}{\lvert}{\rvert}
\bracketfactory{norm}{\Vert}{\Vert}
\bracketfactory{floor}{\lfloor}{\rfloor}
\bracketfactory{ceil}{\lceil}{\rceil}
\bracketfactory{angle}{\langle}{\rangle}

\newcommand{\IE}{\mathop{\mathbb{E}}}

\newcount\minute
\newcount\hour
\newcount\hourMins
\def\now{%
\minute=\time%
\hour=\time \divide \hour by 60%
\hourMins=\hour \multiply\hourMins by 60%
\advance\minute by -\hourMins%
\zeroPadTwo{\the\hour}:\zeroPadTwo{\the\minute}%
}
\def\zeroPadTwo#1{\ifnum #1<10 0\fi#1}

\numberwithin{equation}{section}
\allowdisplaybreaks[4]

\makeatletter

\renewcommand\section{\@startsection {section}{1}{\z@}%
{-3.5ex \@plus -1ex \@minus -.2ex}%
{1.3ex \@plus.2ex}%
{\center\small\sc\mathversion{bold}\MakeUppercase}}

\def\subsection#1{\@startsection {subsection}{2}{0pt}%
{-3.5ex \@plus -1ex \@minus -.2ex}%
{1ex \@plus.2ex}%
{\bf\mathversion{bold}}{#1}}

\def\subsubsection#1{\@startsection{subsubsection}{3}{0pt}%
{\medskipamount}%
{-10pt}%
{\normalsize\itshape}{\kern-2.2ex. #1.}}

\makeatother

\def\blfootnote{\xdef\@thefnmark{}\@footnotetext}

\makeatother

\renewcommand{\cite}{\citet}

\def\^#1{\ifmmode {\mathaccent"705E #1} \else {\accent94 #1} \fi}
\def\~#1{\ifmmode {\mathaccent"707E #1} \else {\accent"7E #1} \fi}

\edef\-#1{\noexpand\ifmmode {\noexpand\bar{#1}} \noexpand\else \-#1\noexpand\fi}
\def\>#1{\vec{#1}}
\def\.#1{\dot{#1}}

\def\wt#1{\widetilde{#1}}
\def\atop{\@@atop}

\renewcommand{\leq}{\leqslant}
\renewcommand{\geq}{\geqslant}
\renewcommand{\phi}{\varphi}
\newcommand{\eps}{\varepsilon}

\newcommand{\eq}{\eqref}

\newcommand{\done}{\mathsf{d}_1}
\newcommand{\dthree}{\mathsf{d}_3}

\newcommand{\bigo}{\mathrm{O}}

\newcommand{\Var}{\mathop{\mathrm{Var}}\nolimits}
\newcommand{\law}{\mathscr{L}}

\def\sp#1{^{(#1)}}

\newcommand{\Cov}{\mathrm{Cov}}
\def\eqd{\stackrel{d}{=}}

\newcommand{\Lip}{\mathrm{Lip}}

\newcommand{\gendim}{\mathfrak{D}}

\makeatletter

\let\originalTilde\~

\def\~#1{\ifmmode {\mathaccent"707E #1} \else {\accent"7E #1} \fi}

\newcounter{ctr}
\loop\stepcounter{ctr}\edef\X{\@Alph\c@ctr}%
  \expandafter\edef\csname s\X\endcsname{\noexpand\mathscr{\X}}
  \expandafter\edef\csname c\X\endcsname{\noexpand\mathcal{\X}}
  \expandafter\edef\csname b\X\endcsname{\noexpand\boldsymbol{\X}}
  \expandafter\edef\csname I\X\endcsname{\noexpand\mathbbm{\X}}  
  \expandafter\edef\csname r\X\endcsname{\noexpand\mathrm{\X}}
\ifnum\thectr<26\repeat

\let\~\originalTilde

\makeatother

\begin{document}

\title{\sc\bf\large\MakeUppercase{
Finite-Dimensional Gaussian Approximation for Deep Neural Networks: Universality in Random Weights }}

\author{
  \begin{tabular}{c@{\hskip 0.5in}c}
 Krishnakumar Balasubramanian &  Nathan Ross\\
\it University of California, Davis & \it University of Melbourne\\
\texttt{kbala@ucdavis.edu} &\texttt{nathan.ross@unimelb.edu.au} \\
\end{tabular}
}
\date{\today}

\maketitle
\begin{abstract}
We study the Finite-Dimensional Distributions (FDDs) of deep neural networks with randomly initialized weights that have finite-order moments. Specifically, we establish Gaussian approximation bounds in the Wasserstein-$1$ norm between the FDDs and their Gaussian limit assuming a Lipschitz activation function and allowing the layer widths to grow to infinity at arbitrary relative rates. In the special case where all widths are proportional to a common scale parameter $n$ and there are $L-1$ hidden layers, we obtain convergence rates of order $n^{-({1}/{6})^{L-1} + \epsilon}$, for any $\epsilon > 0$.



\end{abstract}


\section{Introduction}

\begin{table}
\hspace{-0.26in}    
    \begin{tabular}{|c|c|c|c|c|c|}\hline 
        Source & \makecell{Non-Gaussian\\Weights}   & Activation & Statistic  &  Metric   &  Rate \\ \hline 
         \cite{basteri2024quantitative} & \xmark  & Lip  & FDD & \makecell{$W_2$ \\(w.r.t. \\ 2-norm) } & $n^{-\tfrac{1}{2}}$\\ \hline
         \cite{apollonio2025normal} & \xmark & non-Lip  &  FDD & \makecell{ $W_1$ \\(w.r.t. \\ 2-norm) \\ \& \\ Convex \\ distance}  & $n^{-\tfrac{1}{2}}$ \\ \hline
        \cite{trevisan2023wide} &  \xmark  & Lip  & FDD & \makecell{$W_p$ \\(w.r.t. \\ $p$-norm)  \\ for $p\in [1,\infty)$} & \makecell{$n^{-\tfrac{1}{2}}$ \\ \& \\ $n^{-1}$ \\ (full-rank \\ limiting \\covariance)} \\ \hline
         \cite{Balasubramanian2024}& \cmark  &  Lip &  IDD & \makecell{$W_1$ \\ (w.r.t. \\ sup-norm) }& NA \\ \hline
         \cite{favaro2023quantitative} & \xmark  & \makecell{Poly\\bounded \\ \& Inf \\ diff}   & \makecell{FDD,\\ IDD} & \makecell{Convex\\ distance\\ (FDD) \\ \&\\ \textcolor{black}{$W_2$} \\ \textcolor{black}{(w.r.t)}\\ \textcolor{black}{Sobolev norms)} \\ \textcolor{black}{(IDD)}}  &\makecell{ $n^{-\tfrac{1}{2}}$ \\ (FDD) \\ \&\\ $n^{-\tfrac{1}{8}}$\\ (IDD) \\ (both\\ full-rank \\ limiting \\covariance)}  \\ \hline
         \cite{celli2025entropic}         & \xmark    &\makecell{Poly \\ bounded} & FDD & \makecell{TV \\ \& \\ $W_2$\\(w.r.t. \\ 2-norm)} & \makecell{$n^{-1}$\\(full-rank \\ limiting \\covariance)}\\ \hline
         Theorem~\ref{thm:fddapp}  & \cmark    & Lip  &  FDD& \makecell{$W_1$\\(w.r.t. \\ 2-norm) } & \makecell{$n^{-\tfrac{1}{6^{L-1}}+\epsilon}$ \\ $(\forall \epsilon >0)$} \\ \hline
    \end{tabular}
    \caption{A summary of \emph{quantitative} Gaussian approximation rates for deep neural networks in the proportional width regime (i.e., $n_i \propto n$, $i=1, \ldots, L-1$). FDD and IDD stands for finite-dimensional and infinite-dimensional distributions, respectively. Lip stands for Lipschitz activation functions. NA stands for Not Applicable. We highlight that Theorem~\ref{thm:fddapp} does not assume any full-rank conditions on the limiting covariance.}
    \label{tab:my_label}
\end{table}

An $L$-layer neural network $F\sp{L}$ is a parametric function mapping inputs from a subset $\mathcal{M} \subset \mathbb{R}^{n_0}$—for example, the set of vectorized images containing either cats or dogs—to outputs in $\mathbb{R}^{n_L}$, such as a scalar indicating the likelihood that the image depicts a cat. The input and output dimensions, $n_0$ and $n_L$, are assumed to be fixed and are determined by the data pre-processing pipeline and the desired task output, respectively. The parameters of $F\sp{L}$ are \emph{weight} matrices $(W\sp{\ell})_{\ell=0}^{L-1}$, where $W\sp{\ell}\in \IR^{n_{\ell+1} \times n_\ell}$ and $n_1,\ldots, n_{L-1}$  are referred to as the \emph{widths} of the \emph{hidden layers}. Given these weights and an \emph{activation function} $\sigma:\IR\to\IR$, the deep neural network $F\sp{L}$ is defined recursively through intermediate functions $F\sp{\ell}:\cM\to\IR^{n_\ell}$, for $\ell=1,\ldots, L-1$ by 
\besn{\label{eq:fells}
F\sp{1}(x)& = W\sp{0} x \\ 
F\sp{\ell}(x) &= W\sp{\ell-1} \sigma\bclr{F\sp{\ell-1}(x)}, 
\,\,\, \ell=2,\ldots, L,
}
where $\sigma$ is applied coordinate-wise. Typically, each layer also includes Gaussian \emph{bias} parameters; however, setting them to zero does not affect our results and is therefore omitted—see Remark~\ref{rem:bgone}. Given \emph{training data}, the weights of the neural network are chosen by minimizing a loss function using a gradient descent algorithm with a random initialization, where  the weights $W_{i,j}\sp{\ell}$ are chosen to be 
centered, identically distributed, and independent across $i,j,\ell$.   In the \emph{wide} regime where $n_i$ is large for all $i=1,\ldots,L-1$, it is known 
that the neural network at such an initialization is close to a Gaussian process.
This was first observed by
\citet{neal1996priors}, who showed the convergence of a single hidden-layer neural network at initialization to a Gaussian. The Gaussian limit in the single layer case is fairly straightforward because $F\sp{2}$ can be written as a sum of independent functions, and so classical techniques for sums of independent random variables apply. Subsequent works such as \citet{lee2018deep} and \citet{Matthews2018} provided heuristic arguments and empirical evidence suggesting that deep neural networks with multiple layers similarly exhibit Gaussian behavior as their widths grow. More recently, \citet{Hanin2023} rigorously established asymptotic convergence results for deep networks, further supporting the Gaussian behavior in the infinite-width regime.
Our main result establishes Gaussian approximation bounds, in the Wasserstein-1 distance, between the finite-dimensional distributions (FDDs) of wide neural networks, and these Gaussian process limits, under  general independent weight distributions satisfying mild moment conditions, and assuming a Lipschitz activation function.

Apart from their intrinsic probabilistic interest, there are several compelling reasons to study deep neural networks (DNNs) with non-Gaussian weights. In practice, weights are typically initialized randomly—often with independent and identically-distributed entries—and subsequently optimized using algorithms like stochastic gradient descent (SGD) to minimize a loss function over a given dataset. While Gaussian initializations are widely used, other schemes such as uniform initialization~\citep{glorot2010understanding,golikov2022nongaussian} and Bernoulli distributions, particularly in quantized networks~\citep{li2017training}, are also common. Moreover, in transfer learning settings where pre-trained models are fine-tuned on new tasks, the initial weight distribution is often far from Gaussian. This motivates the study of DNN behavior under more general random weight initializations. Furthermore, \citet{Jacot2018} showed that in the so-called lazy training regime, infinitely wide DNNs exhibit deterministic training dynamics governed by the Neural Tangent Kernel (NTK), and their fluctuations converge to Gaussian processes—emphasizing how the choice of initialization fundamentally shapes both the limiting kernel and the resulting generalization behavior~\citep{carvalho2023wide}. Finally, several works have explored DNNs with weights drawn from heavy-tailed or stable distributions, which may exhibit infinite variance~\citep{der2005beyond, favaro2023deep, bordino2023infinitely, soto2024wide, jung2023stable, loria2023posterior}. These studies are motivated by empirical advantages observed in such models and have established that, in the infinite-width limit, network outputs may converge to stable processes rather than Gaussian processes.


Turning to quantitative Gaussian approximation bounds, results in this literature can be sorted over a number of (categorical) dimensions. One is whether they apply to \emph{shallow} NNs, which have only one hidden layer ($L=2$), or to \emph{deep} NNs, with at least two hidden layers ($L>2$). Another is whether the distributional convergence results apply to univariate distributions $F\sp{L}(x)$ for $x\in \cM$, finite dimensional distributions (FDDs) $(F\sp{L}(x_a))_{a=1}^s$ for a sample $(x_a)_{a=1}^s$ of $\cM$, or at the \emph{functional} level, considering convergence of the \emph{process} $(F\sp{L}(x))_{x\in \cM}$, viewed as an element of some some functional space. A third dimension is the metric used to quantify closeness, and a fourth is the assumptions on the activation $\sigma$.  Finally, and most relevant to this work, is assumptions on the distribution of the weights. In particular, it is almost uniformly assumed across the literature that the weights are Gaussian, notable exceptions being~~\cite{Hanin2023} and \cite{Balasubramanian2024}.

A recent and very thorough account of this literature is given in 
\citet[Section~1.7]{celli2025entropic}, but we give a quick review to contextualize our results. A summary of quantitative bounds for deep neural networks is provided in Table~\ref{tab:my_label}. For shallow NNs, some functional results include \cite{Eldan2021}, \cite{Klukowski2022}, and \cite{Cammarota2024}, all with fairly restrictive assumptions on the activations or weight distributions, and typically in weak $L^p$-type metrics at the functional level. For deep NNs, rates of convergence for FDDs assuming Gaussian weights have been given in  \cite{apollonio2025normal}, \cite{basteri2024quantitative}, \cite{bordino2024non}, \cite{celli2025entropic}, \cite{favaro2023quantitative}, and \cite{trevisan2023wide}, under a variety of activations and metrics. At the functional level, assuming Gaussian weights,  \cite{favaro2023quantitative} derive rates of convergence with respect to Sobolev topologies in the Wasserstein-$2$ norm, and, assuming the activation is infinitely smooth, with respect to the supremum topology. Assuming \emph{general weights}, process   convergence with respect to the supremum topology was shown by \cite{Hanin2023}, and rates of convergence in the same topology in the Wasserstein-$1$ norm  are given in \cite{Balasubramanian2024}. 

\subsection{Main result}

In this work, we provide a bound on the Wasserstein-$1$ norm between the FDDs of $F\sp{L}$ and their Gaussian limit assuming a general weight distribution and a Lipschitz activation function. The bound can be made explicit  in all parameters, is independent of spectral properties of the limiting covariance, and  tends to zero as the widths tend to infinity in any way, for fixed $L$. This is the first such bound appearing in the literature; see also Table~\ref{tab:my_label}.

In order to state our result and discuss it further, we need to introduce the limiting Gaussian distribution. First, 
the proper limiting scaling assumes that there are constants $c_w\sp\ell>0$ for $\ell=0,\ldots, L-1$ such that 
\ben{\label{eq:wellvar}
\Var(W_{ij}\sp{\ell})=\frac{c_w\sp{\ell}}{n_{\ell}}.
}
Given  
this, \cite{Hanin2023} showed that for an absolutely continuous activation with polynomially-bounded derivative and centered independent and identically-distributed weights $W_{ij}\sp\ell$ with all absolute moments $p\geq 1$ bounded uniformly of order $n_\ell^{-p/2}$, the neural network at initialization converges to a
 Gaussian field  which is defined recursively via intermediate fields $G\sp{\ell}$, $\ell=1,\ldots, L-1$ as follows. We first set $G\sp{1}=F\sp{1}$, and, given we have defined the distribution of $G\sp{\ell}$, we let $G\sp{\ell+1}$ be a centered Gaussian random field with covariance 
\ben{\label{eq:gpcov}
C_{ij}\sp{\ell+1}(x,y) = \delta_{ij}\bbbclr{c_w\sp{\ell}\IE\bbcls{\sigma\bclr{G\sp{\ell}_1(x)} \sigma\bclr{G\sp{\ell}_1(y)}}},
}
where $\delta_{ij}$ denotes the Kronecker delta, and $G_1\sp\ell$ is the first component of $G\sp\ell$.  
Note that the covariance function of $F\sp{L}$, calculated from~\eqref{eq:fells} and~\eqref{eq:wellvar}, follows the same recursion, but with $G$ replaced by $F$, so that the coordinates of $F\sp{\ell}$ are uncorrelated, though not independent. An important component of most proofs in the deep regime is that if $F\sp{\ell}$ is approximately Gaussian, then the coordinates should be ``close'' to independent.

Our goal is 
to bound the Wasserstein-$1$ distance
of the finite dimensional distributions between $F\sp{L}$ and those of its limit $G\sp{L}$.
To this end, for a subset of points in the domain  $\chi:=(x_1, x_2, \ldots, x_s)\in \clr{\cM}^s$
and a field $F:\cM\to\IR^d$, denote $F(\chi):= \bclr{F(x_1),\ldots, F(x_s)}$, which  we encode as a matrix in $\IR^{d \times s}$ (due to the different roles of the sample and the vector coordinates), but we equip with the usual Euclidean $2$-norm $\norm{\cdot}_2$.
We can then
 write 
\be{
F\sp{\ell+1}(\chi)=W\sp\ell \sigma\bclr{F\sp{\ell}(\chi)},
}
where $\sigma$ is applied component-wise.
We want to bound the Wasserstein distance
\be{
\done \bclr{F\sp{L}(\chi), G\sp{L}(\chi)}:= \sup_{\zeta \in \cW_{n_L\times s}} \babs{\IE\bcls{\zeta\clr{F\sp{L}(\chi)}}-\IE\bcls{\zeta \clr{G\sp{L}(\chi)}}},
}
where, for any $d\geq 1$, we define
\be{
 \cW_{d\times s} \coloneqq   
 \left\{ \zeta:\IR^{d\times s}  \rightarrow \mathbb{R} : \sup_{f \not = g}\frac{\abs{\zeta(f)-\zeta(g)}}{\norm{f - g}_2}\leq 1 \right\},
}
is the set of $1$-Lipschitz functions on $\IR^{d\times s}$ with respect to $\norm{\cdot}_2$. Note the dimensionality in $\done$ is not explicit, but this will not cause confusion.

Our main result is the following.
\begin{theorem}\label{thm:fddapp}
Let $F\sp{L}$ be the DNN defined at~\eqref{eq:fells} with  centered weights $W_{i,j}\sp\ell$ satisfying~\eqref{eq:wellvar} which are independent across $i,j,\ell$
with identically distributed rows:
$(W_{i,k}\sp\ell)_{k=1}^{n_\ell}\stackrel{d}{=}(W_{j,k}\sp\ell)_{k=1}^{n_\ell}$, and with a  Lipschitz activation~$\sigma$.
Assume further that there is  $p>2$ and constants  $c_{2p}\sp\ell\geq 1$ for  $\ell=0,\ldots, L-2$  such that $\IE\bcls{(W_{ij}\sp\ell)^{2p}}\leq c_{2p}\sp{\ell}/n_\ell^p$ and a constant  $c_3\sp{L-1}$ such that  $\IE\bcls{\abs{W_{ij}\sp{L-1}}^3} \leq c_3\sp{L-1} n_{L-1}^{-3/2}$. Then for $\chi= (x_1,\ldots,x_s)\in \cM^s$, there is a constant $\mathtt{C}$ depending on $\sigma$, $p$, $L$, $\chi$, $c_{2p}\sp\ell$, $\ell=0,\ldots,L-2$ and $c_3\sp{L-1}$ such that 
\be{
\done\bclr{F\sp{L}(\chi), G\sp{L}(\chi)}
    \leq \mathtt{C} \, n_{L}^{1/3}  \sum_{m=1}^{L-1} n_m^{-\frac{1}{6}\left(\frac{p-2}{3(2p-1)}\right)^{L-m-1}},
}
where $G\sp{L}$ is the Gaussian process defined by the covariance recursion~\eqref{eq:gpcov}.
\end{theorem}
\begin{remark}
    The constant $\mathtt{C}$ depends on $\sigma$ through its Lipschitz constant and value at zero, and it depends on $\chi$ through $\{\norm{x_a}_1\}_{a\in\chi}$. It is essentially  polynomial (with degree growing with depth) in the terms given at~\eqref{eq:b2p},~\eqref{eq:c2p}, and~\eqref{eq:d2p}, and could in principle be extracted from the proof. Regardless of tail or boundedness assumptions on the weights, those terms grow with $p$ due to the first  factor of $c^{L-1}$ in~\eqref{eq:b2p}, which is  a combinatorial constant. 
\end{remark}

\begin{remark}\label{rem:bestrate}
If such $c_{2p}\sp{\ell}$ exist for arbitrarily large $p$, then the bound can be made of order
\be{
n_{L}^{1/3}  \sum_{m=1}^{L-1} \frac{1}{n_m^{\left(1/6\right)^{L-m}-\delta}},
}
for any $\delta>0$.  
Considering rates of convergence for the central limit theorem, a presumably optimal rate (in the layer widths) would be of order 
\ben{\label{eq:bestratep}
n_L^{\phi}  \sum_{m=1}^{L-1} n_m^{-\frac{1}{2}},
}
for some $\phi>0$. 
Our worse power stems from  initially working in an integral probability metric~$\dthree$ defined at~\eqref{eq:d3def}, which is weaker than $\done$, but  can be easily used with Stein's method and does not require any assumptions on the covariance matrix of the limiting Gaussian vector $G\sp{L}(\chi)$ (see the forthcoming Remark~\ref{rem:cov}), without which the rate may in truth be worse than the classical setting. We then use a smoothing argument to get back to $\done$ which introduces a factor of $1/3$ to the power. (There is a further factor of $1/2$ introduced for reasons too technical to explain in a remark, but it stems from~\eqref{eq:d3cg2bd3}.) 
The argument is inductive in the layers, and the smoothing is performed at each step, which leads to the power of $(L-m)$ in the factor of $1/6$ in the power. Similar to \citet[Section~6.2]{Balasubramanian2024}, the rate could be somewhat improved by assuming the activation function has three bounded derivatives. The form of the rate would still have the power of $(L-m)$ in the exponent, but now with a base of $(1/2)$   (stemming from~\eqref{eq:d3cg2bd3}) rather than $(1/6)$.
\end{remark}

\begin{remark}\label{rem:cov}
It is worth noting that our result requires no assumption on the limiting covariance of $G\sp{L-1}(\chi)$. As previously mentioned, this is because we first work in the weaker $\dthree$ metric; see Theorem~\ref{thm:dthreestn}. 
The potential degeneracy of the covariance figures prominently in Gaussian approximation results for DNNs, and many rates depend on the eigenvalues of the limiting covariance, including demanding they are all strictly positive (see Table~\ref{tab:my_label}). This is because when the weights are Gaussian, $F\sp{L}$ is conditionally Gaussian given $F\sp{L-1}$, and comparing a conditional Gaussian to $G\sp{L}$ requires comparing the (conditional) covariances; see \citet[(2.7),(2.1), and Lemma~3.4]{basteri2024quantitative} and \citet[Propositions~5.11 and~5.12]{favaro2023quantitative}; which involves the spectrum of covariance matrices. Applying such results may require additional hypotheses on the NN, and $\sigma$ especially. 
In contrast, our constants do not involve eigenvalues of the covariance. 
Revisiting the discussion of the best achievable rates of Remark~\ref{rem:bestrate}, since 
DNNs do not fit neatly into the setting of the classical central limit theorem, the rate~\eqref{eq:bestratep} may not be achievable without additional assumptions.
\end{remark}

\begin{remark}\label{rem:bgone}
    As is typical with approximation results for DNNs in transportation metrics, the Gaussian biases that usually appear in the definition~\eqref{eq:fells} of the NN do not affect the error in our Theorem~\ref{thm:fddapp}.
    See \citet[Proof of Theorem~1.2]{Balasubramanian2024} for the argument in the more challenging infinite-dimensional setting. As discussed in \citet[Section~1.4]{Hanin2023}, the biases must be Gaussian to have a Gaussian limit.
\end{remark}

\begin{remark}
Recent work by~\cite{lee2023deep} and \cite{apollonio2025simulating} shows that the standard assumption of independent weights can be relaxed by introducing \emph{per-node variance mixtures}, allowing the outgoing weights from a neuron to share dependence through a latent variance parameter while still admitting a meaningful infinite-width limit. In particular, when the latent mixing distribution is non-trivial, the resulting limit is no longer a single Gaussian process but a \emph{mixture of Gaussian processes}, exhibiting potentially heavy-tailed behavior. In our setting, one could in principle adopt a similar mixture-based relaxation, though this would substantially complicate the analytical tractability and alter the nature of the limiting kernel object, and consequently, the associated convergence rates. A detailed investigation is left as future work.
\end{remark}

The proof of the theorem starts from the fact (which is commonly used in this setting) that  conditional on $F\sp{L-1}$, $F\sp{L}$ 
is a sum of independent and centered random elements, and so should be unconditionally close to a (mixed) Gaussian with a random covariance (and is exactly a Gaussian mixture in the case the weight matrices are Gaussian). A bound then follows by quantifying how close the Gaussian mixture is to a Gaussian having deterministic covariance, which is the mean of the mixture.
Since the mixture is in terms of $F\sp{L-1}$, there is necessarily an induction step that leverages that 
certain statistics of $F\sp{L-1}$ behave like their $G\sp{L-1}$ analogs, which are easy to handle because the coordinates of $G\sp{L-1}$ are independent. The statistics we need to control are captured  by the relatively simple
Corollary~\ref{cor:gauswt2gaus} below, which 
has an especially nice form because we first work in $\dthree$.

\section{The proof}
Our approach to bounding the Wasserstein distance
is to apply  Stein's method to bound  the weaker integral probability metric
\ben{\label{eq:d3def}
\dthree \bclr{F\sp{L}(\chi), G\sp{L}(\chi)}:= \sup_{\zeta \in \cF_{n_L \times s}} \babs{\IE\bcls{\zeta\clr{F\sp{L}(\chi)}}-\IE\bcls{\zeta \clr{G\sp{L}(\chi)}}},
}
where, for any $d\geq 1$, we define $\cF_{d\times s}$ to be the set of test functions $\zeta: \IR^{d\times s}  \rightarrow \mathbb{R}$ such that for any $\alpha, \beta,\gamma \in \Gamma_{d\times s}:= \{1,\ldots, d\}\times \{1,\ldots, s\}$ and $f\in \IR^{d\times s}$, we have
\be{
\abs{\partial_{\alpha} \eta(f) }\leq 1,
\ \ \abs{\partial_{\alpha,\beta}  \eta(f) }\leq 1, \, \text{ and } \,
\abs{\partial_{\alpha,\beta,\gamma} \eta(f)} \leq 1, 
}
where $\partial_{(\cdot)}$ denotes the partial derivative with respect to the coordinate(s) given by $(\cdot)$.
We then move back to the Wasserstein distance using standard smoothing arguments; see Lemma~\ref{lem:smooth}.

For $F\sp\ell(\chi) = (F_1\sp\ell(\chi),\ldots, F_{n_\ell}\sp\ell(\chi)) \in\IR^{n_\ell\times s}$ and any $1\leq d \leq n_\ell$, we 
write $$F_{[d]}\sp\ell(\chi):=(F\sp\ell_1(\chi),\ldots,F\sp\ell_d(\chi))\in\IR^{d\times s}.$$
Our first step is to look at the distance between an  intermediate field. 
Let $(\wt W\sp\ell)_{\ell=0}^{L-1}$
be independent Gaussian weight matrices having the same variances as $( W\sp\ell)_{\ell=0}^{L-1}$ given at~\eqref{eq:wellvar}, with the two sequences of weight matrices independent so that, e.g., $\wt W\sp{\ell-1}$ is independent of $F\sp{\ell-1}$.
Let $\wt F\sp\ell= \wt W\sp{\ell-1} \sigma( F\sp{\ell-1})$. 
Then the triangle inequality implies for an $d=1,\ldots, n_L$, we have
\ben{
\dthree(F_{[d]}\sp{L}(\chi), G_{[d]}\sp{L}(\chi))  \leq \dthree(F_{[d]}\sp{L}(\chi), \wt F_{[d]}\sp{L}(\chi))
 +\dthree(\wt F_{[d]}\sp{L}(\chi),G_{[d]}\sp{L}(\chi)). \label{eq:unit1} 
}

\subsection{Error for general to Gaussian weights}
For the first term of~\eq{eq:unit1}, we are comparing $ W\sp{L-1} \sigma( F\sp{L-1}(\chi))$ and $\wt W\sp{L-1} \sigma( F\sp{L-1}(\chi))$, and, in anticipation of conditioning on $F\sp{L-1}$, we state and prove the following lemma bounding the $\dthree$ distance in this setting when thinking of $\sigma(F\sp{L-1})$ as fixed.

\begin{lemma}\label{lem:fddapprox}
Let  $h\in \IR^{m\times s}$  and $W$ be a
$d\times m$ random matrix which has centered independent entries having the same variance 
$\Var(W_{ij})=:c_2/m$, also satisfying $\IE\bcls{\abs{W_{ij}}^3}\leq c_3/m^{3/2}$.
Define $F\in \IR^{d\times s}$ by
\be{
F = W h.
}
Analogously, let $\wt W$ be a
$d\times m$ random matrix that has independent centered Gaussian entries with the same variances as $W_{ij}$,
and define $\wt F\in \IR^{d \times s}$ by 
\be{
\wt F = \wt W h.
} 
Then 
\be{
\dthree(F,\wt F)\leq 
 \bbbclr{ \frac{  2 c_2^{3/2} + c_3 }{6}}    \frac{d}{m^{3/2}}\sum_{k=1}^m \bbbclr{\sum_{a=1}^s \abs{h_{k,a}}}^3.
}
\end{lemma}

\begin{proof}[Proof of Lemma~\ref{lem:fddapprox}]
The matrix $\wt F\in \IR^{d \times s}$ is a centered Gaussian, and for $\alpha=(i,a)\in \Gamma_{d\times s}$ and $\beta=(j,\beta)\in \Gamma_{d\times s}$, the covariance is 
\be{
\IE\bcls{\wt F_{i,a},  \wt F_{j,b}} =  \sum_{k,l=1}^m \IE\bcls{\wt W_{ik} \wt W_{j l}} h_{k,a} h_{l,b} = \delta_{ij} c_2 \bbbbclr{\frac{1}{m} \sum_{k=1}^m h_{k,a} h_{k,b}}=: \delta_{ij}\wt C_{a,b}.
}
Following the usual development of Stein's method for the multivariate normal  given by Theorem~\ref{thm:dthreestn} below,
it is enough to
bound the absolute value of 
\ben{\label{eq:stnfddtbd1}
\IE\bcls{ \cA_{\wt C} \, \eta(F)}:= 
\IE \bbbcls{\sum_{i=1}^d\sum_{a,b=1}^s \wt C_{a,b}  \, \partial_{(i,a),(i,b)} \eta(F)}
 - 
 \IE \bbbcls{\sum_{i=1}^d\sum_{a=1}^s F_{i,a}\partial_{(i,a)} \eta(F)},
}
where $\eta$ satisfies 
\ben{\label{eq:pdbds}
\abs{\partial_{\alpha} \eta(f) }\leq 1,
\ \ \abs{\partial_{\alpha,\beta}  \eta(f) }\leq 1/2, \, \text{ and } \,
\abs{\partial_{\alpha,\beta,\gamma} \eta(f)} \leq 1/3,
}
for all $f$.
Working on 
\ben{\label{eq:linterm}
\IE \bbbcls{\sum_{i=1}^d\sum_{a=1}^s F_{i,a} \partial_{(i,a)} \eta(F)}= 
 \sum_{i=1}^d\sum_{a=1}^s \sum_{k=1}^m \IE\bcls{W_{ik} h_{k,a} \partial_{(i,a)} \eta(F)},
}
let $\mathbf{e}_{i,b}\in\IR^{d\times s}$ be the unit matrix with a one in position $(i,b)$, and zeroes elsewhere, and write 
\be{
V_{ik}:= W_{ik}\sum_{b=1}^{s} e_{i,b} h_{k,b}.
}
Note that $W_{ik}$ is independent of $F-V_{ik}$.
Using this independence, and that the mean of $W_{ik}$ is zero implies that the terms from~\eqref{eq:linterm} can be written as
\ben{\label{eq:tt3456}
\IE\bcls{W_{ik} h_{k,a} \partial_{(i,a)} \eta(F)} =\IE\bcls{W_{ik} h_{k,a} \bclr{\partial_{(i,a)} \eta(F)-\partial_{(i,a)} \eta(F-V_{ik})}}.
}
Taylor expanding, we find
\ba{
\partial_{(i,a)} &\eta(F)-\partial_{(i,a)} \eta(F-V_{ik}) \\
	&= \sum_{b=1}^s W_{ik}h_{k,b} \partial_{(i,a),(i,b)}\eta(F-V_{ik})  \\
    &\qquad +\sum_{b=1}^s W_{ik}h_{k,b} \int_0^1 \bclr{ \partial_{(i,a),(i,b)}\eta(F-V_{ik}+sV_{ik}) - \partial_{(i,a),(i,b)}\eta(F-V_{ik})} \mathrm{Leb}(ds).
	}
Plugging this in~\eqref{eq:tt3456} and then~\eqref{eq:linterm}, and using that $\IE[W_{ik}^2]=c_2/m$ and the independence of $W_{ik}$ and $F-V_{ik}$, we can bound the absolute value of~\eqref{eq:stnfddtbd1} as
\ba{
\babs{&\IE \bcls{ \cA_{\wt C} \, \eta(F)}}\\ 
&\leq \bbbabs{\IE \bbbcls{\sum_{i=1}^d\sum_{a,b=1}^s \wt C_{a,b} \bclr{ \partial_{(i,a),(i,b)} \eta(F)-\partial_{(i,a),(i,b)} \eta(F-V_{ik})}}} \\
&+\bbbabs{\sum_{i=1}^d\sum_{a,b=1}^s \sum_{k=1}^m\IE\bcls{W_{ik}^2 h_{k,a}h_{k,b}\int_0^1 \bclr{ \partial_{(i,a),(i,b)}\eta(F-V_{ik}+sV_{ik}) - \partial_{(i,a),(i,b)}\eta(F-V_{ik})} \mathrm{Leb}(ds)}}\\
&\leq \frac{1}{3} \sum_{i=1}^d \sum_{a,b,c=1}^s \IE\babs{\wt C_{a,b} h_{k,c} W_{ik} }+ \frac{1}{6} \sum_{i=1}^d \sum_{a,b,c=1}^s \IE \babs{W_{ik}^3 h_{k,a}h_{k,b}h_{k,c}} \\
&\leq \bbbclr{ \frac{  2 c_2^{3/2} + c_3 }{6}}    \frac{d}{m^{3/2}}\sum_{k=1}^m \sum_{a,b,c=1}^s \abs{h_{k,a}h_{k,b}h_{k,c}},
}
where in the second to last inequality we use the bounds on the derivatives given at~\eqref{eq:pdbds} (for both terms),  and the final inequality uses Cauchy-Schwarz with $\IE[ W_{ik}^2]=c_2/m$ and 
that $\IE[\abs{W_{ik}}^3] \leq c_3/m^{3/2}$.
The result now easily follows.
\end{proof}

We can use Lemma~\ref{lem:fddapprox} to prove the following corollary bounding the first term of~\eq{eq:unit1}.
\begin{corollary}\label{cor:gwt2gaus}
Retaining the notation above, if $\IE\bcls{\abs{W_{ij}\sp{L-1}}^3} \leq c_3\sp{L-1} n_{L-1}^{-3/2}$, then
\ben{\label{eq:dft1}
\dthree(F_{[d]}\sp{L}(\chi),\wt F_{[d]}\sp{L}(\chi))
	\leq  \bbbclr{ \frac{  2 (c_2\sp{L-1})^{3/2} + c_3\sp{L-1} }{6}}    \frac{d \, s^{2}}{\sqrt{n_{L-1}} } \sum_{a=1}^s \IE \bbcls{\babs{\sigma(F_1\sp{L-1}(x_a)}^3}.
}
\end{corollary}

\begin{remark}\label{rem:flmom}
For this bound and other similar terms to come, note that 
Lemma~\ref{lem:momcont} below shows that the moments of $\sigma(F_1\sp{L-1}(x))$ are bounded by a constant depending only on fixed parameters and appropriate moment bounds of the weights. In particular, this term is of order (in the layer widths) $n_{L-1}^{-1/2}$.
\end{remark}

\begin{proof}[Proof of Corollary~\ref{cor:gwt2gaus}]
Note that from the tower
property of conditional expectation, we have
\ben{\label{eq:1}
\dthree(F_{[d]}\sp{L}(\chi),\wt F_{[d]}\sp{L}(\chi))
\leq
\IE \bbcls{\dthree\bbclr{\law\bclr{F_{[d]}\sp{L}(\chi)| F\sp{L-1}(\chi)},\law\bclr{\wt F_{[d]}\sp{L}(\chi)| F\sp{L-1}(\chi)}}},
}
and given 
 $F\sp{L-1}(\chi)$, $F_{[d]}\sp{L}(\chi)$ and $\wt F_{[d]}\sp{L}(\chi)$ have
 the form of $F$ and $\wt F$ with $h_{k,a}:=\sigma(F_k\sp{L-1}(x_a))$ in Lemma~\ref{lem:fddapprox}. The lemma implies that 
 \ba{
\dthree\bbclr{\law\bclr{F_{[d]}\sp{L}(\chi)| F\sp{L-1}(\chi)},&\law\bclr{\wt F_{[d]}\sp{L}(\chi)| F\sp{L-1}(\chi)}} \\
&\leq 
 \bbbclr{ \frac{  2 (c_2\sp{L-1})^{3/2} + c_3\sp{L-1} }{6}}    \frac{d }{n_{L-1}^{3/2} } \sum_{k=1}^{n_{L-1}}\bbbclr{\sum_{a=1}^s \babs{\sigma(F_k\sp{L-1}(x_a)}}^3,\\
 &\leq 
 \bbbclr{ \frac{  2 (c_2\sp{L-1})^{3/2} + c_3\sp{L-1} }{6}}   \frac{d \, s^2 }{n_{L-1}^{3/2}}\sum_{k=1}^{n_{L-1}} \sum_{a=1}^s \babs{\sigma(F_k\sp{L-1}(x_a)}^3,
}
where the last inequality uses
that for $x\in\IR^s$, $\norm{x}_1\leq s^{2/3} \norm{x}_3$, which follows from H\"older's inequality.The result follows after plugging this in to~\eqref{eq:1} and taking expectation, noting that $F\sp{L-1}_i \eqd F\sp{L-1}_j$, because the weight matrices have identically-distributed rows.
\end{proof}

\subsection{Error for Gaussian weights to Gaussian}
Here we discuss how to bound the second term of~\eqref{eq:unit1}, which is the $\dthree$ distance between the neural network $\wt F\sp{L}$ having Gaussian weights, and the Gaussian limit. As before, we start with a general lemma that applies when conditioning on the  previous layer $F\sp{L-1}$.  

\begin{lemma}\label{lem:condgtog}
Let  $H\in \IR^{m\times s}$ be a random matrix with identically-distributed rows, independent of the matrix 
$\wt W\in \IR^{d\times m}$, which has  i.i.d.\ centered independent Gaussian entries having variance 
$\Var(\wt W_{ij})=: c_2/m$.
Define $\wt F\in \IR^{d \times s}$ by
\be{
 \wt F = \wt W H.
} 
Let $G\in \IR^{d\times s}$ be a Gaussian random matrix with centered entries having covariance 
\be{
\IE[ G_{i,a} G_{j,b}] = \delta_{ij} c_2 \IE[K_{a} K_{b}],
}
for some random vector $K\in \IR^{s}$.
Then 
\be{
\dthree(\wt F, G)\leq 
 \frac{ c_2 \, d  }{2} \sum_{a,b=1}^s \bbclr{\babs{\IE[K_a K_b] - \IE[ H_{1,a} H_{1,b}]} + \sqrt{\frac{\Var(H_{1,a}H_{1,b})}{m}}+ \sqrt{\Cov(H_{1,a}H_{1,b},H_{2,a}H_{2,b})}}.
 }
\end{lemma}
\begin{proof}
Again applying the usual development of Stein's method for the multivariate normal  given by Theorem~\ref{thm:dthreestn} below, it's enough to upper bound the absolute value of
\ben{\label{eq:stnfddtbd2}
\IE \bcls{ \cA \, \eta(\wt F)}:= 
 \IE \bbbcls{\sum_{i=1}^d\sum_{a,b=1}^s c_2 \IE[K_a K_b]  \, \partial_{(i,a),(i,b)} \eta(\wt F)}
 - 
 \IE \bbbcls{\sum_{i=1}^d\sum_{a=1}^s \wt F_{i,a}\partial_{(i,a)} \eta(\wt F)},
}
where $\eta$ satisfies 
\ben{\label{eq:pdbdsagain}
\abs{\partial_{\alpha} \eta(f) }\leq 1,
\ \ \abs{\partial_{\alpha,\beta}  \eta(f) }\leq 1/2, \, \text{ and } \,
\abs{\partial_{\alpha,\beta,\gamma} \eta(f)} \leq 1/3.
}
Conditional on $H$, $\wt F$ is Gaussian with conditional covariance
\be{
\IE[F_{i,a} F_{j,b}| H]= \delta_{i,j} \frac{c_2}{m} \sum_{k=1}^m H_{k,a} H_{k,b}=:\delta_{i,j} C_{a,b}(H).
}
Stein's lemma then implies that
\ben{\label{eq:linterm2}
\IE \bbbcls{\sum_{i=1}^d\sum_{a=1}^s \wt F_{i,a} \partial_{(i,a)} \eta(\wt F) \Big| H}
    =\IE \bbbcls{\sum_{i=1}^d\sum_{a,b=1}^s C_{a,b}(H)  \, \partial_{(i,a),(i,b)} \eta(\wt F) \Big | H }.
}
Taking expectation in~\eqref{eq:linterm2}, plugging into~\eqref{eq:stnfddtbd2}, and using~\eqref{eq:pdbdsagain} implies
\ba{
\babs{\IE \bcls{ \cA \, \eta(\wt F)}}
    \leq \frac{1}{2}\sum_{i=1}^d \sum_{a,b=1}^s \bbclr{ \babs{c_2 \IE[K_a K_b]- \IE[C_{a,b}(H)] } +\IE \babs{C_{a,b}(H) - \IE[C_{a,b}(H)] } }. 
}
Using that the rows of $H$ are identically-distributed, we easily find
\be{
\IE[C_{a,b}(H)]= c_2 \IE[H_{1,a}H_{1,b}],
}
and the Cauchy-Schwarz inequality implies 
\ba{
\IE \babs{C_{a,b}(H) - \IE[C_{a,b}(H)] }
&\leq \sqrt{\Var(C_{a,b}(H))}  \\
&    = c_2 \bbbclr{\frac{\Var(H_{1,a}H_{1,b})}{m} + \bclr{1-\tfrac{1}{m}} \Cov(H_{1,a}H_{1,b}, H_{2,a}H_{2,b})}^{1/2},
}
which gives the result.
\end{proof}

Taking $H_{k,a}=\sigma(F_k\sp{L-1}(x_a))$ and $K_a=\sigma(G_1\sp{L-1}(x_a))$ in Lemma~\ref{lem:condgtog} (noting that $H$ has identically distributed rows as a consequence of the same property for the weight matrices) easily gives the following corollary bounding the second term of~\eq{eq:unit1}.
\begin{corollary}\label{cor:gauswt2gaus}
In the notation above, we have 
\ban{
\dthree\bclr{&\wt F_{[d]}\sp{L}(\chi), G_{[d]}\sp{L}(\chi)} \notag \\
	&\leq  \frac{ c_2\sp{L-1} \, d  }{2} \sum_{a,b=1}^s \bbbclc{\bbabs{\IE\bcls{\sigma(G_1\sp{L-1}(x_a))\sigma(G_1\sp{L-1}(x_b))} - \IE\bcls{\sigma(F_1\sp{L-1}(x_a))\sigma(F_1\sp{L-1}(x_b))}} \label{eq:d3cg2bd1} \\
    &\qquad\qquad + \sqrt{\frac{\Var\bclr{\sigma(F_1\sp{L-1}(x_a))\sigma(F_1\sp{L-1}(x_b)}}{n_{L-1}}} \label{eq:d3cg2bd2}\\
    &\qquad\qquad+ \sqrt{\Cov\bclr{\sigma(F_1\sp{L-1}(x_a))\sigma(F_1\sp{L-1}(x_b)),\sigma(F_2\sp{L-1}(x_a))\sigma(F_2\sp{L-1}(x_b)}}}.
    \label{eq:d3cg2bd3}
 }    
\end{corollary}

As already mentioned in Remark~\ref{rem:flmom}, the moments of $\sigma(F_1\sp{L-1}(x))$ are of constant order in the widths, so the terms corresponding to~\eqref{eq:d3cg2bd2}  are bounded of order $n_{L-1}^{-1/2}$. 
For~\eqref{eq:d3cg2bd1}, the term will be close to zero if $\bclr{F_{1}\sp{L-1}(x_a), F_1\sp{L-1}(x_b)}\stackrel{d}{\approx}\bclr{G_{1}\sp{L-1}(x_a), G_1\sp{L-1}(x_b)}$ and we have some control over the moments of both. 
The former is shown by an inductive argument, noting that $n_{L-1}$ being large has no effect on this term.
For the same reason, we expect~\eqref{eq:d3cg2bd3} to be close to the same covariance with $F$'s replaced by $G$'s, but that will be equal to zero since the coordinates of $G\sp{L-1}$ are independent. To formalize this argument, we need the following lemma which controls the difference of moments of random vectors in terms of the the $\done$ distance between them, and their moments.

\begin{lemma}\label{lem:bdmoms1}
Let $X,Z\in \IR^d$ be random vectors with $\IE[\abs{X_i}^{q+d}]<C_q$ and $\IE[\abs{Z_i}^{q+d}]<C_q$ for $i=1,\ldots, d$ and some $q> 0$. Then
\be{
\bbbabs{\IE\prod_{i=1}^d X_i - \IE\prod_{i=1}^d Z_i} \leq 2 \sqrt{d} (4\sqrt{d} \, C_q)^{(d-1)/(q+d-1)} \done\clr{X,Z}^{q/(q+d-1)}.
}
\end{lemma}
\begin{proof}
Denoting the indicator function by $\II$, for any coupling $(X,Z)$ of $X$ and $Z$, writing the product as a telescoping sum, and using the triangle inequality and a union-type bound gives
\ba{
\bbbabs{\prod_{i=1}^d X_i - \prod_{i=1}^d Z_i}	 
&\leq \bbbclr{\sum_{j=1}^d \abs{X_j-Z_j}   \prod_{i=1}^{j-1}\abs{ Z_i} \prod_{k=j+1}^d \abs{X_k}}
\II\bbcls{\abs{Z_j}\leq T, \abs{X_j}\leq T, \, \, \mathrm{ for } \,\, \mathrm{ 
all } \,\, j=1,\ldots, d} \\
& \qquad+ \bbbabs{\prod_{i=1}^d X_i - \prod_{i=1}^d Z_i}\bbbclr{\sum_{j=1}^d\bbclr{ \II\bcls{\abs{Z_j} \geq \rT} +\II\bcls{\abs{X_j} \geq \rT}}}	\\
	&\leq \rT^{d-1} \sum_{j=1}^d \abs{X_j-Z_j}  \\
	&\qquad+\bbbclr{\prod_{i=1}^d \abs{X_i} + \prod_{i=1}^d \abs{Z_i}} \bbbclr{\sum_{j=1}^d\bbclr{ \II\bcls{\abs{Z_j} \geq \rT} +\II\bcls{\abs{X_j} \geq \rT}}}.
}
Each of the $2\times(2d)=4d$ terms in the second term will be bounded in the same way. Taking expectation in one representative term, we have, using H\"older's inequality,
\ba{
\IE\bbbclc{ \prod_{i=1}^d \abs{X_i}\II\bcls{\abs{Z_j} \geq \rT} } 
	&\leq \rT^{-q}\IE\bbbclc{ \abs{Z_j}^q\prod_{i=1}^d \abs{X_i}}\\
	&\leq  \rT^{-q} \IE\bcls{ \abs{Z_j}^{q+d}}^{q/(q+d)}  \prod_{i=1}^d\IE\bcls{ \abs{X_i}^{q+d}}^{1/q+d}\\
	&\leq \rT^{-q} C_q.
}
Thus, choosing the Wasserstein-optimal coupling of $(X,Z)$ and using that the Euclidean $L^1$ distance is
upper bounded by $\sqrt{d}$ times the $L^2$ distance, we have shown
\be{
\bbbabs{\IE \prod_{i=1}^d X_i - \IE \prod_{i=1}^d Z_i} \leq \sqrt{d} \rT^{d-1} \done(X,Z) + \rT^{-q} (4d \, C_q).
}
Choosing 
\be{
\rT:= \left(\frac{4\sqrt{d} \, C_q}{\done(X,Z)}\right)^{1/(q+d-1)} 
}
gives the result.
\end{proof}

In order to apply Lemma~\ref{lem:bdmoms1} with Corollary~\ref{cor:gauswt2gaus}, we need to control the moments of $\sigma(F_1\sp\ell(x))$, which is given in the next lemma.
\begin{lemma}\label{lem:momcont}
Fix $p\in \IN$. If for  $\ell=0,\ldots, L-2$, there are constants $c_{2p}\sp\ell\geq 1$ such that $\IE\bcls{(W_{ij}\sp\ell)^{2p}}\leq c_{2p}\sp{\ell}/n_\ell^p$, then 
\be{
\max\bbclc{\IE\bcls{\bclr{\sigma(F_1\sp{L-1}(x))}^{2p}}, \IE\bcls{\bclr{\sigma(G_1\sp{L-1}(x))}^{2p}} } \leq
B_{2p}\sp{L-1}(x),
}
where 
\ben{\label{eq:b2p}
B_{2p}\sp{L-1}(x):= c^{L-1} \bclc{( c_{2p}\sp{0}\norm{x}_1^{2p}/n_0^p )\vee 1}\bclc{(\Lip_\sigma+\abs{\sigma(0)})^{2p(L-1)}\vee 1 }\prod_{\ell=1}^{L-2}c_{2p}\sp{\ell},
}
and $c$ is a constant depending only on $p$.
\end{lemma}

\begin{proof}
Comparing to $\sigma(0)$, using the triangle inequality for the $(2p)$-moment norm, and that $\sigma$ is Lipschitz, we find
\besn{\label{eq:trimombd}
\IE\bcls{\bclr{\sigma(F_1\sp{L-1}(x))}^{2p}} ^{1/(2p)}
    &=\IE\bbcls{\bclr{ \sigma(F_1\sp{L-1}(x))-\sigma(0) + \sigma(0) }^{2p}}^{1/(2p)}\\
    &\leq \mathrm{Lip}_\sigma \IE\bcls{(F_1\sp{L-1}(x))^{2p}}^{1/(2p)} +\abs{\sigma(0)}. 
}
To compute the $(2p)$th moment of $F_1\sp{L-1}(x)$, 
let $A_{n_{L-2}}\sp{2p}$ be the set of $(j_1,\ldots,j_{2p})\in \{1,\ldots, n_{L-2}\}^{2p}$ where the label of every coordinate appears at least twice. Since the $W_{1j}\sp{L-2}$ are independent and have mean zero, we find
\ba{
\IE\bcls{(F_1\sp{L-1}(x))^{2p}}
	&= \sum_{j_1,\ldots, j_{2p}=1}^{n_{L-2}} \IE\bbbcls{ \prod_{\ell=1}^{2p} W_{1,j_\ell}\sp{L-2}}\IE\bbbcls{\prod_{\ell=1}^{2p} \sigma(F\sp{L-2}_{j_\ell}(x))} \\
	&= \sum_{(j_1,\ldots,j_{2p})\in A_{n_{L-2}}\sp{2p}}  \IE\bbbcls{ \prod_{\ell=1}^{2p} W_{1,j_\ell}\sp{L-2}}\IE\bbbcls{\prod_{\ell=1}^{2p} \sigma(F\sp{L-2}_{j_\ell}(x))} \\
    &\leq \babs{A_{n_{L-2}}\sp{2p}} \frac{c_{2p}\sp{L-2}}{n_{L-2}^p}\IE\bcls{\bclr{\sigma(F_1\sp{L-2}(x))}^{2p}} \\
    &\leq c \cdot c_{2p}\sp{L-2}\IE\bcls{\bclr{\sigma(F_1\sp{L-2}(x))}^{2p}},
    }
where $c\geq1$ is a constant depending only on $p$,  the second to last inequality follows from H\"older's inequality and the bounds on the $(2p)$th moments of $W_{1j}\sp{L-2}$, and the last inequality follows because $\babs{A_{n_{L-2}}\sp{2p}}=\bigo(n_{L-2}^p)$, with a constant depending only on $p$.
Combining this with~\eqref{eq:trimombd}, we have shown that 
\besn{\label{eq:samlo}
\IE\bcls{\bclr{\sigma(F_1\sp{L-1}(x))}^{2p}}
    &\leq  \bbclr{ \mathrm{Lip}_\sigma\bclr{  c \cdot c_{2p}\sp{L-2}\IE\bcls{\bclr{\sigma(F_1\sp{L-2}(x))}^{2p}}}^{1/(2p)} + \abs{\sigma(0)}}^{2p}\\
     &\leq \bclr{\mathrm{Lip}_\sigma  +\abs{\sigma(0)}}^{2p} \bbclc{ \bclr{c \cdot c_{2p}\sp{L-2}\IE\bcls{\bclr{\sigma(F_1\sp{L-2}(x))}^{2p}}}\vee 1 }.
    }
The first assertion now follows by iterating, using that $c \cdot c_{2p}\sp{L-2} \geq 1$, and that 
\besn{\label{eq:basemom}
\IE\bcls{(F_1\sp{1}(x))^{2p}}
	&= \sum_{j_1,\ldots, j_{2p}=1}^{n_{L-2}} \IE\bbbcls{ \prod_{\ell=1}^{2p} W_{1,j_\ell}\sp{0}} \prod_{\ell=1}^{2p} x_{j_\ell}\leq \frac{c_{2p}\sp{0}}{n_0^p}\sum_{j_1,\ldots, j_{2p}=1}^{n_{L-2}} \ \prod_{\ell=1}^{2p} \abs{x_{j_\ell}}=\frac{c_{2p}\sp{0}}{n_0^p} \norm{x}_1^{2p}.
}

For $G\sp{L-1}$, the same calculation as~\eqref{eq:trimombd} implies 
\ba{
\IE\bcls{\bclr{\sigma(G_1\sp{L-1}(x))}^{2p}}^{1/(2p)}
  &\leq \mathrm{Lip}_\sigma \IE\bcls{(G_1\sp{L-1}(x))^{2p}}^{1/(2p)} +\abs{\sigma(0)}. 
}
Since $G_1\sp{L-1}(x)$ is Gaussian, its $(2p)$-moment norm is written in terms of the standard deviation and $(2p)$-moment norm of the standard normal as 
\ba{
\IE\bcls{(G_1\sp{L-1}(x))^{2p}} 
    &= (c_w\sp{L-2})^{p} \IE\bcls{(G_1\sp{L-1}(x))^{2}}^p 2^p \bbbclr{\frac{ \Gamma(p+1/2)}{\Gamma(1/2)}} \\
    &\leq   c \cdot (c_w\sp{L-2})^p \IE\bcls{\sigma(G_1\sp{L-2}(x))^{p}}.
}
Combining the last two displays in the same way as~\eqref{eq:samlo}, and noting that H\"older's inequality implies $(c_w\sp{\ell})^p\leq c_{2p}\sp{\ell}$, we have 
\ba{
\IE\bcls{\bclr{\sigma(G_1\sp{L-1}(x))}^{2p}}
    &\leq (\mathrm{Lip}_\sigma  +\abs{\sigma(0)})^{2p}\bclc{\bclr{ c \cdot c_{2p}\sp{L-2} \IE\bcls{\sigma(G_1\sp{L-2}(x))^{2}}^p } \vee 1}.
}
Iterating, and using~\eqref{eq:basemom} again, noting that $G_1\sp{1}=F\sp{1}$ gives the  bound for $G\sp{L-1}_1$.
\end{proof}

We can combine Lemmas~\ref{lem:bdmoms1} and~\ref{lem:momcont}
to obtain the following result which bounds the terms appearing in Corollary~\ref{cor:gauswt2gaus} in terms of Wasserstein distance and moments.
\begin{lemma}\label{lem:cov2wass}
Retaining the notation above, if there is an integer $p>2$ and constants  $c_{2p}\sp\ell\geq 1$ for  $\ell=0,\ldots, L-2$  such that $\IE\bcls{(W_{ij}\sp\ell)^{2p}}\leq c_{2p}\sp{\ell}/n_\ell^p$, then 
for any $x_a,x_b\in\cM$, $\chi_{a,b}=(x_a,x_b)$, and $B\sp{L-1}_{2p}$ defined in Lemma~\ref{lem:momcont}, we have
\besn{\label{eq:2jtbd}
\bbabs{\IE\bcls{\sigma(G_1\sp{L-1}(x_a))\sigma(G_1\sp{L-1}&(x_b))} - \IE\bcls{\sigma(F_1\sp{L-1}(x_a))\sigma(F_1\sp{L-1}(x_b))}}\\
    &\leq  C_{2p}\sp{L-1}(x_a,x_b) \done\bclr{F\sp{L-1}_{[1]}(\chi_{a,b}), G\sp{L-1}_{[1]}(\chi_{a,b})}^{(2p-2)/(2p-1)},
}
where 
\ben{\label{eq:c2p}
C_{2p}\sp{L-1}(x_a,x_b)=2\sqrt{2}\bclr{4\sqrt{2} \clc{B_{2p}\sp{L-1}(x_a) \vee B_{2p}\sp{L-1}(x_b)} }^{1/(2p-1)} \Lip_\sigma^{(2p-2)/(2p-1)}.
}
In addition, we have
\besn{\label{eq:4jtbd}
\babs{\Cov\bclr{\sigma &(F_1\sp{L-1}(x_a))\sigma(F_1\sp{L-1}(x_b)),\sigma(F_2\sp{L-1}(x_a))\sigma(F_2\sp{L-1}(x_b)}}\\
    &\leq D_{2p}\sp{L-1}(x_a,x_b) \done\bclr{F\sp{L-1}_{[2]}(\chi_{a,b}), G\sp{L-1}_{[2]}(\chi_{a,b})}^{(2p-4)/(2p-1)}\\
& \quad + 2\sqrt{B_2\sp{L-1}(x_a) B_2\sp{L-1}(x_b)}
C_{2p}\sp{L-1}(x_a,x_b) \done\bclr{F\sp{L-1}_{[1]}(\chi_{a,b}), G\sp{L-1}_{[1]}(\chi_{a,b})}^{(2p-2)/(2p-1)}
}
where
\ben{\label{eq:d2p}
D_{2p}\sp{L-1}(x_a,x_b)=4\bclr{8\clc{B_{2p}\sp{L-1}(x_a) \vee B_{2p}\sp{L-1}(x_b)} }^{3/(2p-1)} \Lip_\sigma^{(2p-4)/(2p-1)}.
}
\end{lemma}


\begin{proof}
For the first bound~\eqref{eq:2jtbd}, we apply Lemma~\ref{lem:bdmoms1} with  $d=2$, $q=2p-2$, $X_1 = \sigma(F\sp{L-1}_1(x_a))$, $X_2=\sigma(F_1\sp{L-1}(x_b))$, $Z_1 =\sigma( G\sp{L-1}_1(x_a))$, $Z_2=\sigma(G_1\sp{L-1}(x_b))$, and, using Lemma~\ref{lem:momcont}, with
\be{
C_q = \clc{B_{2p}\sp{L-1}(x_a) \vee B_{2p}\sp{L-1}(x_b)}.
}
This gives the first assertion after noting that the definition of~$\done$ and scaling implies that for any two random vectors $Y$ and $Y'$, 
\ben{\label{eq:donescale}
\done\clr{ \sigma(Y), \sigma(Y')} \leq \Lip_\sigma \done\clr{ Y, Y' }.
}
For the second bound~\eqref{eq:4jtbd}, we note that independence implies
\be{
\Cov\bclr{\sigma(G_1\sp{L-1}(x_a))\sigma(G_1\sp{L-1}(x_b)),\sigma(G_2\sp{L-1}(x_a))\sigma(G_2\sp{L-1}(x_b)}=0,
}
and so the triangle inequality implies 
\begin{align}
\begin{split}
\babs{
\Cov\bclr{&\sigma(F_1\sp{L-1}(x_a))\sigma(F_1\sp{L-1}(x_b)),\sigma(F_2\sp{L-1}(x_a))\sigma(F_2\sp{L-1}(x_b)}} 
\end{split} \notag \\
\begin{split}\label{eq:4tfg}
    &\leq \bbabs{\IE\bcls{\sigma(F_1\sp{L-1}(x_a))\sigma(F_1\sp{L-1}(x_b))\sigma(F_2\sp{L-1}(x_a))\sigma(F_2\sp{L-1}(x_b))} \\
    &\qquad \qquad -\IE\bcls{\sigma(G_1\sp{L-1}(x_a))\sigma(G_1\sp{L-1}(x_b))\sigma(G_2\sp{L-1}(x_a))\sigma(G_2\sp{L-1}(x_b))}} 
\end{split} \\
\begin{split} \label{eq:2tfg}
    &\quad + \babs{\IE\bcls{\sigma(G_1\sp{L-1}(x_a))\sigma(G_1\sp{L-1}(x_b))  }- \IE\bcls{\sigma(F_1\sp{L-1}(x_a))\sigma(F_1\sp{L-1}(x_b))}}\\
    &\qquad \qquad \times \bclr{\babs{\IE\bcls{\sigma(G_1\sp{L-1}(x_a))\sigma(G_1\sp{L-1}(x_b))  }}+ \babs{\IE\bcls{\sigma(F_1\sp{L-1}(x_a))\sigma(F_1\sp{L-1}(x_b))}}}.
\end{split}
\end{align}
We can bound~\eqref{eq:4tfg} using  Lemma~\ref{lem:bdmoms1} with  $d=4$, $q=2p-4$, 
$X_1,X_2, Z_1, Z_2$ as above, and $X_3,X_4, Z_3, Z_4$ 
defined as $X_1,X_2, Z_1, Z_2$, but with second components instead of first, and using Lemma~\ref{lem:momcont}, with
\be{
C_q = \clc{B_{2p}\sp{L-1}(x_a) \vee B_{2p}\sp{L-1}(x_b)}.
}
Then, using again the scaling relation~\eqref{eq:donescale}, we have that~\eqref{eq:4tfg}
is bounded by
\besn{\label{eq:4bdt1f}
&4\bclr{8\clc{B_{2p}\sp{L-1}(x_a) \vee B_{2p}\sp{L-1}(x_b)} }^{3/(2p-1)} \Lip_\sigma^{(2p-4)/(2p-1)} \\
& \qquad \times  \done\bclr{F\sp{L-1}_{[2]}(\chi_{a,b}), G\sp{L-1}_{[2]}(\chi_{a,b})}^{(2p-4)/(2p-1)}. 
}

Finally, the difference in the first term of~\eqref{eq:2tfg} is the same as~\eqref{eq:2jtbd}, and each of the summands in the second term can be bounded using Lemma~\ref{lem:bdmoms1} and Cauchy-Schwarz, leading to the upper bound on~\eqref{eq:2tfg}  of
\besn{\label{eq:4bdt2f}
4\sqrt{2}\sqrt{B_2(x_a) B_2(x_b)}&\bclr{4\sqrt{2} \clc{B_{2p}\sp{L-1}(x_a) \vee B_{2p}\sp{L-1}(x_b)} }^{1/(2p-1)} \Lip_\sigma^{(2p-2)/(2p-1)} \\
    &  \times \done\bclr{F\sp{L-1}_{[1]}(\chi_{a,b}), G\sp{L-1}_{[1]}(\chi_{a,b})}^{(2p-2)/(2p-1)}. 
}
Combining~\eqref{eq:4bdt1f} and~\eqref{eq:4bdt2f} implies~\eqref{eq:4jtbd}.
\end{proof}

Combining Corollary~\ref{cor:gauswt2gaus} and Lemmas~\ref{lem:momcont} and~\ref{lem:cov2wass} gives the following corollary bounding the second term of~\eqref{eq:unit1}.
\begin{corollary}\label{cor:gauswt2guassdone}
Retaining the notation above, assume there is an integer $p>2$ and constants  $c_{2p}\sp\ell\geq 1$ for  $\ell=0,\ldots, L-2$  such that $\IE\bcls{(W_{ij}\sp\ell)^{2p}}\leq c_{2p}\sp{\ell}/n_\ell^p$, and for any $a,b\in\{1,\ldots, s\}$ write  $\chi_{a,b}=(x_a,x_b)$. Then 
\ba{
\dthree\bclr{\wt F_{[d]}\sp{L}(\chi), G_{[d]}\sp{L}(\chi)} &\leq  \frac{ c_2\sp{L-1} \, d  }{2} \sum_{a,b=1}^s \bbbclc{
 C_{2p}\sp{L-1}(x_a,x_b) \done\bclr{F\sp{L-1}_{[1]}(\chi_{a,b}), G\sp{L-1}_{[1]}(\chi_{a,b})}^{\frac{(2p-2)}{(2p-1)}}\\
    &\hspace{2mm} + \frac{\bclr{B_{2p}\sp{L-1}(x_a)B_{2p}\sp{L-1}(x_b)}^{1/(2p)}}{\sqrt{n_{L-1}}} \\
    &\hspace{2mm}+ \sqrt{2C_{2p}\sp{L-1}(x_a,x_b) \sqrt{B_2(x_a) B_2(x_b)}} \done\bclr{F\sp{L-1}_{[1]}(\chi_{a,b}), G\sp{L-1}_{[1]}(\chi_{a,b})}^{\frac{(p-1)}{(2p-1)}}\\
    &\hspace{2mm}+ \sqrt{D_{2p}\sp{L-1}(x_a,x_b)}\done\bclr{F\sp{L-1}_{[2]}(\chi_{a,b}), G\sp{L-1}_{[2]}(\chi_{a,b})}^{(p-2)/(2p-1)}},
}
where $B_{2p}\sp{L-1}, C_{2p}\sp{L-1}, D_{2p}\sp{L-1}$ are defined at~\eqref{eq:b2p},~\eqref{eq:c2p}, and~\eqref{eq:d2p}.
\end{corollary}

We can now prove our main approximation result.
\begin{proof}[Proof of Theorem~\ref{thm:fddapp}]
We start with~\eqref{eq:unit1}, and apply Corollaries~\ref{cor:gwt2gaus} and~\ref{cor:gauswt2guassdone}, and Lemma~\ref{lem:momcont} along with H\"older's inequality to find
\ban{
&\hspace{5mm}\dthree\bclr{F\sp{L}(\chi), G\sp{L}(\chi)} \notag\\ \label{eq:1stbd} 
& \leq   \frac{ c_2\sp{L-1} \, n_L }{2} \sum_{a,b=1}^s \bbbclc{
 C_{2p}\sp{L-1}(x_a,x_b) \done\bclr{F\sp{L-1}_{[1]}(\chi_{a,b}), G\sp{L-1}_{[1]}(\chi_{a,b})}^{\frac{(2p-2)}{(2p-1)}} \notag\\
    &\hspace{2mm} + \frac{\bclr{B_{2p}\sp{L-1}(x_a)B_{2p}\sp{L-1}(x_b)}^{1/(2p)}}{\sqrt{n_{L-1}}} \notag \\
    &\hspace{2mm}+\sqrt{2 C_{2p}\sp{L-1}(x_a,x_b) \sqrt{B_2(x_a) B_2(x_b)} }\done\bclr{F\sp{L-1}_{[1]}(\chi_{a,b}), G\sp{L-1}_{[1]}(\chi_{a,b})}^{(p-1)/(2p-1)} \notag \\
    &\hspace{2mm}+  \sqrt{D_{2p}\sp{L-1}(x_a,x_b)} \done\bclr{F\sp{L-1}_{[2]}(\chi_{a,b}), G\sp{L-1}_{[2]}(\chi_{a,b})}^{(p-2)/(2p-1)}}
  \notag  \\
 &\qquad  +  \frac{n_L \,s^2}{\sqrt{n_{L-1}} }  \bbbclr{ \frac{  2 (c_2\sp{L-1})^{3/2} + c_3\sp{L-1} }{6}}  \sum_{a=1}^s \IE \bbcls{\babs{\sigma(F_1\sp{L-1}(x_a)}^3}\notag \\
 & \leq   \frac{ c_2\sp{L-1} \, n_L  }{2} \sum_{a,b=1}^s \bbbclc{
 C_{2p}\sp{L-1}(x_a,x_b) \done\bclr{F\sp{L-1}_{[1]}(\chi_{a,b}), G\sp{L-1}_{[1]}(\chi_{a,b})}^{(2p-2)/(2p-1)}\notag \\
    &\hspace{2mm}+ \sqrt{2 C_{2p}\sp{L-1}(x_a,x_b) \sqrt{B_2(x_a) B_2(x_b)}} \done\bclr{F\sp{L-1}_{[1]}(\chi_{a,b}), G\sp{L-1}_{[1]}(\chi_{a,b})}^{(p-1)/(2p-1)}\notag \\
    &\hspace{2mm}+ \sqrt{ D_{2p}\sp{L-1}(x_a,x_b)} \done\bclr{F\sp{L-1}_{[2]}(\chi_{a,b}), G\sp{L-1}_{[2]}(\chi_{a,b})}^{(p-2)/(2p-1)}}
   \notag \\
 & \qquad +   \frac{n_L \,s^2}{\sqrt{n_{L-1}} }\bbbclc{ \bbbclr{ \frac{  2 (c_2\sp{L-1})^{3/2} + c_3\sp{L-1} }{6}}  \sum_{a=1}^s B_{2p}\sp{L-1}(x_a)^{3/(2p)}+\frac{c_2\sp{L-1}}{2}\bbbclr{\sum_{a=1}^s B_{2p}\sp{L-1}(x_a)}^2 }. \notag
}

To bound the $\done$ terms in the bounds above, we use induction on $\ell=2,\ldots,L-1$, where the induction statement at level $\ell$ is
\ben{\label{eq:indstat}
\done\bclr{F\sp{\ell}_{[d]}(\chi_{a,b}), G_{[d]}\sp{\ell}(\chi_{a,b})}
    \leq \mathtt{C}_\ell \sum_{m=1}^{\ell-1} n_m^{-\frac{1}{6}\left(\frac{p-2}{3(2p-1)}\right)^{\ell-1-m}}, \ \ d=1,2.
}
For $\ell=2$, combining~\eqref{eq:unit1}, Corollary~\ref{cor:gwt2gaus} with $c_3\sp{1}=(c_{2p}\sp{1})^{3/(2p)}$ (using H\"older's inequality) and Corollary~\ref{cor:gauswt2guassdone}, noting that $F\sp{1}=G\sp{1}$, implies 
\ba{
\dthree\bclr{F\sp{2}_{[d]}(\chi_{a,b}), G_{[d]}\sp{2}(\chi_{a,b})}
&\leq \frac{d}{\sqrt{n_{1}} } 
\bbbclr{ \frac{  2 (c_2\sp{1})^{3/2} + c_3\sp{1} }{6}}   \sum_{c\in\{a,b\}} \IE \bbcls{\babs{\sigma(F_1\sp{1}(x_c)}^3}\\
&\qquad\qquad + \frac{c_2\sp{1} d }{2 \sqrt{n_1}}\bbbclr{\sum_{c\in\{a,b\}} B_{2p}\sp{1}(x_c)^{1/(2p)}}^2\\
&\leq \frac{d}{\sqrt{n_{1}} } 
\bbbclr{ \frac{  2 (c_2\sp{1})^{3/2} +6 c_2\sp{1}+ c_3\sp{1} }{6}}   \sum_{c\in\{a,b\}} B_{2p}\sp{1}(x_c)^{3/(2p)},
}
where the last inequality uses Lemma~\ref{lem:momcont}, H\"older's inequality, and that $B_{2p}\sp{1}(x) \geq 1$. Assuming $n_1$ is large enough that this bound is less than $2$ (only depends on $\chi$, $\sigma$ and weight moments), the smoothing Lemma~\ref{lem:smooth} with $k=2 d\leq 4$ implies
\ba{
\done\bclr{F\sp{2}_{[d]}(\chi_{a,b}), G_{[d]}\sp{2}(\chi_{a,b})}
    &\leq \frac{2^{8/3}}{n_{1}^{1/6}}  \bbbclc{
\bbbclr{ \frac{  2 (c_2\sp{1})^{3/2} +6c_2\sp{1}+ c_3\sp{1} }{6}}   \sum_{c\in\{a,b\}} B_{2p}\sp{1}(x_c)^{3/(2p)}}^{1/3}.
}
This establishes~\eqref{eq:indstat} for $\ell=2$.

Now, assuming~\eqref{eq:indstat} holds for some $\ell\in\{2,\ldots, L-2\}$, we show it holds for $(\ell+1)$. 
Just as in the $\ell=2$ case, we start with~\eqref{eq:unit1} and apply Corollaries~\ref{cor:gwt2gaus} and~\ref{cor:gauswt2guassdone}, and Lemma~\ref{lem:momcont} to find
\besn{\label{eq:indstpd3}
\dthree&\bclr{F\sp{\ell+1}_{[d]}(\chi_{a,b}), G_{[d]}\sp{\ell+1}(\chi_{a,b})}\\
&\leq \frac{ c_2\sp{\ell} \, d  }{2} \bbbclc{
 C_{2p}\sp{\ell}(x_a,x_b)  \done\bclr{F\sp{\ell}_{[1]}(\chi_{a,b}), G\sp{\ell}_{[1]}(\chi_{a,b})}^{(2p-2)/(2p-1)}\\
 &\hspace{20mm}+ \sqrt{2 C_{2p}\sp{\ell}(x_a,x_b) \sqrt{B_2\sp\ell(x_a) B_2\sp\ell(x_b)}  }\done\bclr{F\sp{\ell}_{[1]}(\chi_{a,b}), G\sp{\ell}_{[1]}(\chi_{a,b})}^{(p-1)/(2p-1)}\\
    &\hspace{20mm}+ \sqrt{ D_{2p}\sp{\ell}(x_a,x_b) }\done\bclr{F\sp{\ell}_{[2]}(\chi_{a,b}), G\sp{\ell}_{[2]}(\chi_{a,b})}^{(p-2)/(2p-1)}}\\
&\qquad \qquad +\frac{d}{\sqrt{n_{\ell}} } 
\bbbclr{ \frac{  2 (c_2\sp{\ell})^{3/2}+ 6 c_2\sp\ell+ c_3\sp{\ell} }{6}}    \sum_{c\in\{a,b\}} B_{2p}\sp{\ell}(x_c)^{3/(2p)} \\
&\leq c_2\sp{\ell}  \bbbclc{
 C_{2p}\sp{\ell}(x_a,x_b) +\sqrt{2 C_{2p}\sp{\ell}(x_a,x_b)\sqrt{B_2\sp\ell(x_a) B_2\sp\ell(x_b)}} + \sqrt{D_{2p}\sp{\ell}(x_a,x_b)}} \\
 &\hspace{25mm} \times\bbclr{\mathtt{C}_\ell\sum_{m=1}^{\ell-1} n_m^{-\frac{1}{6}\left(\frac{p-2}{3(2p-1)}\right)^{\ell-1-m}}}^{(p-2)/(2p-1)}\\
&\qquad \qquad +\frac{2}{\sqrt{n_{\ell}} } 
\bbbclr{ \frac{  2 (c_2\sp{\ell})^{3/2}+ 6 c_2\sp\ell+ c_3\sp{\ell} }{6}}    \sum_{c\in\{a,b\}} B_{2p}\sp{\ell}(x_c)^{3/(2p)},
}
where we have used the induction hypothesis, and assumed that $n_1,\ldots, n_{\ell-1}$ are large enough so that
\be{
\mathtt{C}_\ell\sum_{m=1}^{\ell-1} n_m^{-\frac{1}{6}\left(\frac{p-2}{3(2p-1)}\right)^{\ell-1-m}}\leq 1,
}
which allows us to use the smallest power $(p-2)/(2p-1)$.
Now assuming $n_1,\ldots, n_\ell$ are large enough so the final bound in~\eqref{eq:indstpd3} is less than two, we 
apply the  smoothing Lemma~\ref{lem:smooth} with $\mathcal{D}=2 d\leq 4$ to find
\be{
\done\bclr{F\sp{\ell+1}_{[d]}(\chi_{a,b}), G_{[d]}\sp{\ell+1}(\chi_{a,b})}
    \leq \mathtt{C}_{\ell+1} \sum_{m=1}^{\ell} n_m^{-\frac{1}{6}\left(\frac{p-2}{3(2p-1)}\right)^{\ell-m}},
}
which completes the induction.

Finally, using the bound in~\eqref{eq:indstat} for $\ell=L-1$ and using the smoothing lemma (again assuming $n_1,\ldots,n_{L-1}$ large enough for its application, and to simplify the number of powers) we obtain
\be{
\done\bclr{F\sp{L}(\chi), G\sp{L}(\chi)}
    \leq \mathtt{C} n_{L}^{1/3}  \sum_{m=1}^{L-1} n_m^{-\frac{1}{6}\left(\frac{p-2}{3(2p-1)}\right)^{L-m-1}},
}
as desired.
\end{proof}

\subsection{Stein's method and smoothing results}

Our implementation of Stein's method follows from the following theorem, which is by now standard. It can easily be read from
\citet[Lemmas~1 and~2]{Meckes2009}, and the fundamental ideas go back to \cite{Gotze1991}. See also \cite{Chatterjee2008}, \cite{Chen2011normal}, \citet[Chapter~12]{Goldstein1996}, \citet[Chapter~4]{nourdin2012normal}, and \cite{Reinert2009}.

\begin{theorem}\label{thm:dthreestn}
Let $Z\in \IR^{\gendim}$ be a centered multivariate normal random variable with covariance matrix 
$\Sigma:=(\sigma_{ij})_{i,j=1}^{\gendim}$ and define the operator on twice-differential functions $\eta:\IR^{\gendim}\to\IR$ by 
\be{
\cA_{\Sigma}\eta(x)
 := \sum_{i,j=1}^{\gendim} \sigma_{ij} 
 \partial_{i,j}\eta(x)
 - \sum_{i=1}^{\gendim} x_i \partial_i \eta(x),
}
where we write $\partial_{(\cdot)}$ for partial derivatives with respect to the coordinates given by $(\cdot)$.
Then for any integrable random vector $X\in \IR^{\gendim}$, we have
\be{
\dthree(X,Z) \leq \sup_{\eta\in \wt\cF_{\gendim}} \babs{\IE \bcls{\cA_\Sigma \eta (X)}},
}
where 
$\wt\cF_{\gendim}$ consists of functions $\eta:\IR^{\gendim}\to\IR$ with three derivatives satisfying, for $i,j,k\in\{1,\ldots,\gendim\}$,
\be{
\norm{\partial_{i} \eta}_\infty\leq 1,
\ \ \norm{\partial_{i,j}  \eta }_\infty\leq 1/2, \, \text{ and } \,
\norm{\partial_{i,j,k} \eta}_{\infty} \leq 1/3. 
}
\end{theorem}

The next lemma records the outcome of a 
standard Gaussian smoothing argument to bound 
the~$\done$ metric in terms of the $\dthree$ 
metric. See~\citet[Section~4.2]{Raic2018}.
\begin{lemma}[Smoothing lemma]\label{lem:smooth}
If $U$ and $V$  are random vectors in $\IR^{\gendim}$ with $\dthree(U,V)\leq 2\sqrt{{\gendim}}$, then
\be{
\done(U,V) \leq 
    2 (2\sqrt{{\gendim}})^{2/3} \bclr{\dthree(U,V)}^{1/3}.
}
\end{lemma}
\begin{proof}
Let $\zeta\in \cW_{\gendim}$ and $S\in\IR^{\gendim}$ be a vector with i.i.d.\ standard Gaussian entries. For $\eps\in(0,1)$,  define the regularization 
    \be{
    \zeta_\eps(u):= \IE[ \zeta(u + \eps S)] = \int \zeta(t) \prod_{i=1}^{\gendim} \phi_{u_i, \eps^2}(t_i) dt,
    }
    where $\phi_{\mu,\sigma^2}$ denotes the density of a one dimensional standard normal with mean $\mu$ and variance $\sigma^2$. Clearly $\zeta_\eps$ is infinitely differentiable, and, by a straightforward calculation, we have the explicit formulas
    \ban{
    \frac{\partial}{\partial u_i} \zeta_\eps(u) &= \frac{1}{\eps}\IE[S_i  \zeta(u + \eps S)],  \label{eq:partsm1}\\
    \frac{\partial^2}{\partial u_i\partial u_j} \zeta_\eps(u) &=
    \begin{cases}
        \frac{1}{\eps^2}\IE[S_i S_j \zeta(u + \eps S)], & i \not=j \\
        \frac{1}{\eps^2}\IE[(S_i^2 -1) \zeta(u + \eps S)], & i=j.
            \end{cases}\label{eq:partsm2}
    }
On the other hand, because $\zeta$ is $1$-Lipschitz, we have 
\be{
\abs{\zeta_\eps(u) -\zeta_\eps(v)} \leq \IE\bcls{\abs{\zeta(u+\eps S)-\zeta(v+\eps S)}}\leq \norm{u-v}_2,
}
which implies that
\be{
\bbabs{\frac{\partial}{\partial u_i} \zeta_\eps(u)}\leq 1.
}
Using that $\zeta$ is $1$-Lipschitz again and the expression for the first derivative at~\eqref{eq:partsm1} implies that 
\be{
\bbabs{ \frac{\partial}{\partial u_i} \zeta_\eps(u)- \frac{\partial}{\partial u_i} \zeta_\eps(v)} \leq \frac{1}{\eps} \IE\bcls{\abs{S_i}}\norm{u-v}_2, 
}
which implies that 
\be{
\bbabs{\frac{\partial^2}{\partial u_i\partial u_j} \zeta_\eps(u)}\leq \frac{1}{\eps}.
}
Using the same argument again and~\eqref{eq:partsm2} gives
\be{
\bbabs{\frac{\partial^3}{\partial u_i\partial u_j\partial u_{k}} \zeta_\eps(u)}\leq \frac{1}{\eps^2},
}
and we can conclude that $ \eps^2\zeta_\eps \in \cF_{\gendim}$. (Note also these three inequalities follow directly from \citet[Lemma~4.6]{Raic2018}.)

To estimate $\done(U,V)$, we compare the expectations of a generic test function $\zeta\in \cW_{\gendim}$, and then add and subtract expectations against  $\zeta_\eps$ to get to $\dthree(U,V)$, as follows.
\ba{
\babs{\IE\cls{\zeta(U)}-\IE\cls{\zeta(V)}} 
    &\leq\babs{\IE\cls{\zeta(U)}-\IE\cls{\zeta_\eps(U)}}+\babs{\IE\cls{\zeta(V)}-\IE\cls{\zeta_\eps(V)}} +\babs{\IE\cls{\zeta_\eps(U)}-\IE\cls{\zeta_\eps(V)}} \\
    & \leq 2\eps \IE\bcls{\norm{S}_2} + \frac{\dthree(U,V)}{ \eps^2} \leq 2\sqrt{{\gendim}} \eps + \frac{\dthree(U,V)}{\eps^2},
}
where the second inequality uses that $\zeta$ is $1$-Lipschitz and that that $\eps^2\zeta_\eps \in \cF_{\gendim}$. Taking the supremum over $\zeta \in  \cW_{\gendim}$ implies the same inequality holds with left-hand side replaced by $\done(U,V)$. 

If $\dthree(U,V)\leq 2\sqrt{\gendim}$, we may choose $\eps=(\dthree(U,V)/ 2\sqrt{\gendim})^{1/3}\leq 1$ to find 
\be{
\done(U,V) \leq 
    2 (2\sqrt{\gendim})^{2/3} \bclr{\dthree(U,V)}^{1/3},
}
as desired.
\end{proof}

\textbf{Acknowledgment.} We thank the anonymous referees and Associate Editor, as well as the Editor, for their helpful comments, which greatly improved the paper. Balasubramanian gratefully acknowledges support from National Science Foundation through the grant NSF DMS-2413426. 

\bibliographystyle{abbrvnat}
\bibliography{ref}

\begin{thebibliography}{35}
\providecommand{\natexlab}[1]{#1}
\providecommand{\url}[1]{\texttt{#1}}
\expandafter\ifx\csname urlstyle\endcsname\relax
  \providecommand{\doi}[1]{doi: #1}\else
  \providecommand{\doi}{doi: \begingroup \urlstyle{rm}\Url}\fi

\bibitem[Apollonio et~al.(2025{\natexlab{a}})Apollonio, De~Canditiis, Franzina,
  Stolfi, and Torrisi]{apollonio2025normal}
N.~Apollonio, D.~De~Canditiis, G.~Franzina, P.~Stolfi, and G.~L. Torrisi.
\newblock Normal approximation of random {G}aussian neural networks.
\newblock \emph{Stoch. Syst.}, 15\penalty0 (1):\penalty0 88--110,
  2025{\natexlab{a}}.

\bibitem[Apollonio et~al.(2025{\natexlab{b}})Apollonio, Franzina, and
  Torrisi]{apollonio2025simulating}
N.~Apollonio, G.~Franzina, and G.~L. Torrisi.
\newblock Simulating posterior {B}ayesian neural networks with dependent
  weights.
\newblock \emph{arXiv preprint arXiv:2507.22095}, 2025{\natexlab{b}}.

\bibitem[Balasubramanian et~al.(2024)Balasubramanian, Goldstein, Ross, and
  Salim]{Balasubramanian2024}
K.~Balasubramanian, L.~Goldstein, N.~Ross, and A.~Salim.
\newblock Gaussian random field approximation via {S}tein's method with
  applications to wide random neural networks.
\newblock \emph{Appl. Comput. Harmon. Anal.}, 72:\penalty0 Paper No. 101668,
  27, 2024.

\bibitem[Basteri and Trevisan(2024)]{basteri2024quantitative}
A.~Basteri and D.~Trevisan.
\newblock Quantitative {G}aussian approximation of randomly initialized deep
  neural networks.
\newblock \emph{Machine Learning}, 113\penalty0 (9):\penalty0 6373--6393, 2024.

\bibitem[Bordino et~al.(2023)Bordino, Favaro, and
  Fortini]{bordino2023infinitely}
A.~Bordino, S.~Favaro, and S.~Fortini.
\newblock Infinitely wide limits for deep stable neural networks: {S}ub-linear,
  linear and super-linear activation functions.
\newblock \emph{arXiv preprint arXiv:2304.04008}, 2023.

\bibitem[Bordino et~al.(2024)Bordino, Favaro, and Fortini]{bordino2024non}
A.~Bordino, S.~Favaro, and S.~Fortini.
\newblock {Non-asymptotic approximations of {G}aussian neural networks via
  second-order Poincar{\'e} inequalities}.
\newblock In \emph{Proceedings of the 6th Symposium on Advances in Approximate
  Bayesian Inference}, volume~1, page~34, 2024.

\bibitem[Cammarota et~al.(2024)Cammarota, Marinucci, Salvi, and
  Vigogna]{Cammarota2024}
V.~Cammarota, D.~Marinucci, M.~Salvi, and S.~Vigogna.
\newblock A quantitative functional central limit theorem for shallow neural
  networks.
\newblock \emph{Mod. Stoch. Theory Appl.}, 11\penalty0 (1):\penalty0 85--108,
  2024.

\bibitem[Carvalho et~al.(2023)Carvalho, Costa, Mour{\~a}o, and
  Oliveira]{carvalho2023wide}
L.~Carvalho, J.~L. Costa, J.~Mour{\~a}o, and G.~Oliveira.
\newblock Wide neural networks: From non-{G}aussian random fields at
  initialization to the {NTK} geometry of training.
\newblock \emph{arXiv preprint arXiv:2304.03385}, 2023.

\bibitem[Celli and Peccati(2025)]{celli2025entropic}
L.~Celli and G.~Peccati.
\newblock Entropic bounds for conditionally {G}aussian vectors and applications
  to neural networks.
\newblock \emph{arXiv preprint arXiv:2504.08335}, 2025.

\bibitem[Chatterjee and Meckes(2008)]{Chatterjee2008}
S.~Chatterjee and E.~Meckes.
\newblock Multivariate normal approximation using exchangeable pairs.
\newblock \emph{ALEA Lat. Am. J. Probab. Math. Stat.}, 4:\penalty0 257--283,
  2008.

\bibitem[Chen et~al.(2011)Chen, Goldstein, and Shao]{Chen2011normal}
L.~H. Chen, L.~Goldstein, and Q.-M. Shao.
\newblock \emph{Normal approximation by {S}tein's method}.
\newblock Springer, 2011.

\bibitem[de~G.~Matthews et~al.(2018)de~G.~Matthews, Hron, Rowland, Turner, and
  Ghahramani]{Matthews2018}
A.~G. de~G.~Matthews, J.~Hron, M.~Rowland, R.~E. Turner, and Z.~Ghahramani.
\newblock {G}aussian process behaviour in wide deep neural networks.
\newblock In \emph{International Conference on Learning Representations}, 2018.

\bibitem[Der and Lee(2005)]{der2005beyond}
R.~Der and D.~Lee.
\newblock {Beyond {G}aussian processes: On the distributions of infinite
  networks}.
\newblock In \emph{Advances in Neural Information Processing Systems},
  volume~18, 2005.

\bibitem[Eldan et~al.(2021)Eldan, Mikulincer, and Schramm]{Eldan2021}
R.~Eldan, D.~Mikulincer, and T.~Schramm.
\newblock Non-asymptotic approximations of neural networks by {G}aussian
  processes.
\newblock In \emph{Conference on Learning Theory}, pages 1754--1775. PMLR,
  2021.

\bibitem[Favaro et~al.(2023)Favaro, Fortini, and Peluchetti]{favaro2023deep}
S.~Favaro, S.~Fortini, and S.~Peluchetti.
\newblock Deep stable neural networks: {L}arge-width asymptotics and
  convergence rates.
\newblock \emph{Bernoulli}, 29\penalty0 (3):\penalty0 2574--2597, 2023.

\bibitem[Favaro et~al.(2025)Favaro, Hanin, Marinucci, Nourdin, and
  Peccati]{favaro2023quantitative}
S.~Favaro, B.~Hanin, D.~Marinucci, I.~Nourdin, and G.~Peccati.
\newblock Quantitative {CLT}s in deep neural networks.
\newblock \emph{Probab. Theory Related Fields}, 191\penalty0 (3-4):\penalty0
  933--977, 2025.

\bibitem[Glorot and Bengio(2010)]{glorot2010understanding}
X.~Glorot and Y.~Bengio.
\newblock Understanding the difficulty of training deep feedforward neural
  networks.
\newblock In \emph{Proceedings of the 13th International Conference on
  Artificial Intelligence and Statistics}, pages 249--256, 2010.

\bibitem[Goldstein and Rinott(1996)]{Goldstein1996}
L.~Goldstein and Y.~Rinott.
\newblock Multivariate normal approximations by {S}tein's method and size bias
  couplings.
\newblock \emph{J. Appl. Probab.}, 33\penalty0 (1):\penalty0 1--17, 1996.

\bibitem[Golikov and Yang(2022)]{golikov2022nongaussian}
E.~Golikov and G.~Yang.
\newblock Non-{G}aussian tensor programs.
\newblock In \emph{Advances in Neural Information Processing Systems},
  volume~35, 2022.

\bibitem[G{{\"o}}tze(1991)]{Gotze1991}
F.~G{{\"o}}tze.
\newblock On the rate of convergence in the multivariate {CLT}.
\newblock \emph{Ann. Probab.}, 19\penalty0 (2):\penalty0 724--739, 1991.

\bibitem[Hanin(2023)]{Hanin2023}
B.~Hanin.
\newblock Random neural networks in the infinite width limit as {G}aussian
  processes.
\newblock \emph{Ann. Appl. Probab.}, 33\penalty0 (6A):\penalty0 4798--4819,
  2023.

\bibitem[Jacot et~al.(2018)Jacot, Gabriel, and Hongler]{Jacot2018}
A.~Jacot, F.~Gabriel, and C.~Hongler.
\newblock Neural tangent kernel: Convergence and generalization in neural
  networks.
\newblock In \emph{Advances in Neural Information Processing Systems},
  volume~31, 2018.

\bibitem[Jung et~al.(2023)Jung, Lee, Lee, and Yang]{jung2023stable}
P.~Jung, H.~Lee, J.~Lee, and H.~Yang.
\newblock {$\alpha$}-stable convergence of heavy/light-tailed infinitely wide
  neural networks.
\newblock \emph{Advances in Applied Probability}, 55\penalty0 (4):\penalty0
  1415--1441, 2023.

\bibitem[Klukowski(2022)]{Klukowski2022}
A.~Klukowski.
\newblock Rate of convergence of polynomial networks to {G}aussian processes.
\newblock In \emph{Conference on Learning Theory}, pages 701--722. PMLR, 2022.

\bibitem[Lee et~al.(2023)Lee, Ayed, Jung, Lee, Yang, and Caron]{lee2023deep}
H.~Lee, F.~Ayed, P.~Jung, J.~Lee, H.~Yang, and F.~Caron.
\newblock Deep neural networks with dependent weights: Gaussian process mixture
  limit, heavy tails, sparsity and compressibility.
\newblock \emph{Journal of Machine Learning Research}, 24\penalty0
  (289):\penalty0 1--78, 2023.

\bibitem[{L}ee et~al.(2018){L}ee, {S}ohl {D}ickstein, Pennington, Novak,
  Schoenholz, and Bahri]{lee2018deep}
J.~{L}ee, J.~{S}ohl {D}ickstein, J.~Pennington, R.~Novak, S.~Schoenholz, and
  Y.~Bahri.
\newblock Deep neural networks as {G}aussian processes.
\newblock In \emph{International Conference on Learning Representations}, 2018.

\bibitem[Li et~al.(2017)Li, De, Xu, Studer, Samet, and
  Goldstein]{li2017training}
H.~Li, S.~De, Z.~Xu, C.~Studer, H.~Samet, and T.~Goldstein.
\newblock Training quantized nets: A deeper understanding.
\newblock \emph{Advances in Neural Information Processing Systems}, 30, 2017.

\bibitem[Lor{\'\i}a and Bhadra(2023)]{loria2023posterior}
J.~Lor{\'\i}a and A.~Bhadra.
\newblock {Posterior inference on shallow infinitely wide Bayesian neural
  networks under weights with unbounded variance}.
\newblock \emph{arXiv preprint arXiv:2305.10664}, 2023.

\bibitem[Meckes(2009)]{Meckes2009}
E.~Meckes.
\newblock On {S}tein's method for multivariate normal approximation.
\newblock In \emph{High dimensional probability {V}: the {L}uminy volume},
  volume~5 of \emph{Inst. Math. Stat. (IMS) Collect.}, pages 153--178. Inst.
  Math. Statist., Beachwood, OH, 2009.

\bibitem[Neal(1996)]{neal1996priors}
R.~M. Neal.
\newblock Priors for infinite networks.
\newblock In \emph{Bayesian learning for neural networks}, pages 29--53.
  Springer, 1996.

\bibitem[Nourdin and Peccati(2012)]{nourdin2012normal}
I.~Nourdin and G.~Peccati.
\newblock \emph{Normal approximations with {M}alliavin calculus: {F}rom
  {S}tein's method to universality}, volume 192.
\newblock Cambridge University Press, 2012.

\bibitem[Rai\v{c}(2018)]{Raic2018}
M.~Rai\v{c}.
\newblock A multivariate central limit theorem for {L}ipschitz and smooth test
  functions.
\newblock Preprint \url{https://arxiv.org/abs/1812.08268}, 2018.

\bibitem[Reinert and R{{\"o}}llin(2009)]{Reinert2009}
G.~Reinert and A.~R{{\"o}}llin.
\newblock Multivariate normal approximation with {S}tein's method of
  exchangeable pairs under a general linearity condition.
\newblock \emph{Ann. Probab.}, 37\penalty0 (6):\penalty0 2150--2173, 2009.

\bibitem[Soto(2024)]{soto2024wide}
T.~Soto.
\newblock {Wide stable neural networks: Sample regularity, functional
  convergence and Bayesian inverse problems}.
\newblock \emph{arXiv preprint arXiv:2407.03909}, 2024.

\bibitem[Trevisan(2023)]{trevisan2023wide}
D.~Trevisan.
\newblock Wide deep neural networks with {G}aussian weights are very close to
  {G}aussian processes.
\newblock \emph{arXiv preprint arXiv:2312.11737}, 2023.

\end{thebibliography}
\end{document}